\def\year{2022}\relax
\documentclass[letterpaper]{article} 
\usepackage{aaai22}  
\usepackage{times}  
\usepackage{helvet}  
\usepackage{courier}  
\usepackage[hyphens]{url}  
\usepackage{graphicx} 
\usepackage{natbib}  
\usepackage{caption} 
\DeclareCaptionStyle{ruled}{labelfont=normalfont,labelsep=colon,strut=off} 
\usepackage[ruled,noend]{algorithm2e}
\SetKwFor{RepTimes}{repeat}{times}{end}

\usepackage[inline]{enumitem}
\usepackage{amsthm}
\usepackage{todonotes}

\usepackage{amsmath,amsfonts,bm}

\def\eqref#1{equation~\ref{#1}}

\def\1{\bm{1}}

\DeclareMathAlphabet{\mathsfit}{\encodingdefault}{\sfdefault}{m}{sl}
\SetMathAlphabet{\mathsfit}{bold}{\encodingdefault}{\sfdefault}{bx}{n}

\newcommand{\E}{\mathbb{E}}

\DeclareMathOperator*{\argmax}{arg\,max}

\makeatletter
\providecommand*{\barvee}{%
  \mathbin{%
    \mathpalette\@barvee{}%
  }%
}
\newcommand*{\@barvee}[2]{%
  \sbox0{$#1\veebar\m@th$}%
  \sbox2{%
    \hbox to \wd0{%
      \hss
      \resizebox{1.05\wd0}{\height}{$#1-\m@th$}%
      \hss
    }%
  }%
  \sbox4{%
    \resizebox{\wd0}{.7\ht0}{$#1\vee\m@th$}%
  }%
  \sbox6{$#1\vcenter{}$}
  \ht2=\ht6 %
  \vbox to \ht0{%
    \copy2 %
    \vss
    \copy4 %
  }%
}
\makeatother

\newcommand{\state}{\mathcal{S}}
\newcommand{\action}{\mathcal{A}}
\newcommand{\dynamics}{P}
\newcommand{\reward}{R}
\newcommand{\rmax}{\reward_{\text{MAX}}}
\newcommand{\rmin}{\reward_{\text{MIN}}}
\newcommand{\goals}{\mathcal{G}}
\renewcommand{\v}{V}
\newcommand{\q}{Q}
\newcommand{\vpi}{V^\pi}
\newcommand{\qpi}{Q^\pi}
\newcommand{\pistar}{\pi^*}

\newcommand{\vstar}{V^*}
\newcommand{\qstar}{Q^*}

\newcommand{\tasks}{\mathcal{M}}

\usepackage{mathtools}

\newcommand{\rbar}{\bar{\reward}}
\newcommand{\rbarmin}{\rbar_{\text{MIN}}}
\newcommand{\pibar}{\bar{\pi}}
\newcommand{\pibarstar}{\bar{\pi}^*}
\newcommand{\vbar}{\bar{\v}}
\newcommand{\vbarpi}{\bar{\v}^{\pibar}}
\newcommand{\vstarbar}{\bar{\v}^*}
\newcommand{\vstarbara}{\Tilde{\bar{\v}}}
\newcommand{\qbar}{\bar{\q}}
\newcommand{\qbarpi}{\bar{\q}^{\pibar}}
\newcommand{\qstarbar}{\bar{\q}^{*}}

\newcommand{\goalq}{\bar{\mathcal{\q}}^*}

\newcommand{\gstar}{G^*_{s, g, a}}

\makeatletter
\providecommand*{\barvee}{%
  \mathbin{%
    \mathpalette\@barvee{}%
  }%
}
\makeatother

\newtheorem{definition}{Definition}
\newtheorem{theorem}{Theorem}

\newlength{\strutheight}
\usepackage{amsmath}
\usepackage{mathtools}
\usepackage{mathrsfs}
\usepackage{subcaption}

\title{World Value Functions: Knowledge Representation for Learning and Planning}
\author{
    Geraud Nangue Tasse, Benjamin Rosman, Steven James\\
}
\affiliations{
    School of Computer Science and Applied Mathematics\\
    University of the Witwatersrand\\
    Johannesburg, South Africa\\
    geraud.nanguetasse1@students.wits.ac.za, \{benjamin.rosman1, steven.james\}@wits.ac.za \\
}

\begin{document}

\maketitle

\begin{abstract}

We propose world value functions (WVFs), a type of goal-oriented general value function that represents how to solve not just a given task, but any other goal-reaching task in an agent's environment.
This is achieved by equipping an agent with an internal goal space defined as all the world states where it experiences a terminal transition.
The agent can then modify the standard task rewards to define its own reward function, which provably drives it to learn how to achieve all reachable internal goals, and the value of doing so in the current task.
We demonstrate two key benefits of WVFs in the context of learning and planning.
In particular, given a learned WVF, an agent can compute the optimal policy in a new task by simply estimating the task's reward function.
Furthermore, we show that WVFs also implicitly encode the transition dynamics of the environment, and so can be used to perform planning.
Experimental results show that WVFs can be learned faster than regular value functions, while their ability to infer the environment's dynamics can be used to integrate learning and planning methods to further improve sample efficiency.

\end{abstract}

\section{Introduction}
\label{sec:intro}

A grand challenge of artificial intelligence is to create general agents capable of solving a wide variety of tasks in the real world. 
To accomplish this, we require a general decision-making framework that models agents' interaction with the world,  and a sufficiently general representation to capture the knowledge agents acquire. Reinforcement learning (RL) \cite{sutton1998introduction} is one such framework and although it has made several major breakthroughs in recent years, ranging from robotics \citep{levine16} to board games \citep{silver17}, these agents are typically narrowly designed to solve only a single task. 

In RL, tasks are specified through a reward function from which the agent receives feedback. 
Most commonly, an agent represents its knowledge in the form of a value function, representing the sum of future rewards it expects to receive.
However, since the value function is directly tied to one single reward function (and hence task), it is definitionally insufficient for constructing agents capable of solving a wide range of tasks.

In this work, we seek to overcome this limitation by proposing \textit{world value functions} (WVFs), a goal-oriented knowledge representation that encodes how to solve not only the current task, but also any other goal-reaching task. 
In the literature, agents with such abilities are said to possess \textit{mastery} \cite{veeriah2018many}, and we prove that WVFs do, in fact, possess this property in deterministic environments.
Importantly, WVFs are a form of general value function \cite{sutton11} that can be learned from a single stream of experience; no additional information or modifications to the standard RL framework are required.

WVFs have several desirable properties, which we formally prove in the deterministic setting. In particular, we show that (i) given a learned WVF, any new task can be solved by estimating its reward function, which reduces the problem to supervised learning; and (ii) WVFs implicitly encode the dynamics of the world and can be used for model-based RL.
Experimental results in the Four Rooms domain \citep{sutton99} validate our theoretical findings, while demonstrating that not only can WVFs be learned faster than regular value functions, they can also be leveraged to perform Dyna-style planning \citep{sutton90} to improve sample efficiency.

\section{Preliminaries} \label{sec:bg}

We model an agent's environment as a Markov Decision Process (MDP) $(\mathcal{\state}, \action, \dynamics, \reward)$, where 
\begin{enumerate*}[label=(\roman*)]
  \item $\state$ is the state space,
  \item $\action$ is the action space,
  \item $\dynamics(s, a, s')$ are the transition dynamics of the world, and
  \item $\reward$ is a reward function, bounded by $[\rmin,\rmax]$, representing the task the agent must solve.
\end{enumerate*}
Note that in this work, we focus on environments with \textit{deterministic} dynamics, but put no restrictions on their complexity.

The agent's aim is to compute a \textit{policy} $\pi$ from $\state$ to $\action$ that optimally solves a given task.
This is often achieved by learning a value function that represents the expected return obtained under $\pi$ starting from state $s$: $\vpi(s) = \E^\pi \left[ \sum_{t=0}^{\infty} r(s_t, a_t, s_{t+1}) \right]$.
Similarly, the action-value function $\qpi(s, a)$ represents the expected return obtained by executing $a$ from $s$, and thereafter following $\pi$.
The optimal action-value function is given by $\qstar(s, a) = \max_\pi \qpi(s, a)$ for all states $s$ and actions $a$, and the optimal policy follows by acting greedily with respect to $\qstar$ at each state.

\section{World Value Functions} \label{sec:wvfs}

We now introduce \textit{world value functions} (WVFs), which provably encode how to reach all achievable goals.
We first define the internal goal space $\goals \subseteq \state$ of the agent as all states where it experiences a terminal transition.

Different from other goal-oriented approaches where goals are specified by the environment, here the goal an agent wishes to achieve is chosen by itself.
The agent's aim now is to simultaneously solve the current task, while also learning how to achieve its own internal goals. 
To do so, the agent can define its own goal-oriented reward function $\rbar$, which extends $\reward$ to penalise itself for achieving goals it did not intend to:
\[
\rbar(s, g, a, s^\prime) \coloneqq \begin{cases}
\rbarmin & \text{if } g \neq s \text{ and } s^\prime \text{ is absorbing }\\
\reward(s, a, s^\prime) &\text{otherwise},
\end{cases}
\]
where $\rbarmin$ is a large negative penalty that can be derived from the bounds of the reward function \cite{nangue2020boolean}.
Intuitively, the penalty $\rbarmin$ adds one bit of information to the agent's rewards, and we will later prove this is sufficient for the agent to learn how to achieve its internal goals in the current task.

The agent must now compute a \textit{world policy} $\pibar: \state \times \goals \to \text{Pr}(\action)$ that optimally reaches its internal goal states. Given a world policy $\pibar$, the corresponding WVF is defined as 
 $\qbarpi(s, g, a) \coloneqq \E_{s^\prime}^{\pibar} \left[ \rbar(s, g, a, s^\prime) + \vbarpi(s^\prime, g) \right]$, where $ \vbarpi(s, g) \coloneqq \E^{\pibar} \left[ \sum_{t=0}^{\infty} \rbar(s_t, g, a_t, s_{t+1}) \right]$.

Since the WVF satisfies the Bellman equations, $\qstarbar(s,g,a)$ can be learned using any suitable RL algorithm, such as Q-learning (see Algorithm~\ref{algo:qlearn}).

\subsection{Properties of World Value Functions}

While a learned WVF encodes the values of achieving all internal goals, it can still be used to solve the task in which it was learned.
Theorem~\ref{thm:1} below demonstrates that the current task's reward and value function can be recovered by simply maximising over goals:

\begin{theorem}
Let $M=(\mathcal{\state}, \action, \dynamics, \reward)$ be a deterministic task with optimal action-value function $\qstar$ and optimal world action-value function $\qstarbar$. 
Then for all $(s, a, s^\prime)$ in $\state \times \action \times \state$, we have
\begin{enumerate*}[label=(\roman*)]
    \item $\reward(s, a, s^\prime) = \max\limits_{g \in \goals} \rbar(s, g, a, s^\prime)$, and
    \item $\qstar(s, a) = \max\limits_{g \in \goals} \qstarbar(s, g, a)$.
\end{enumerate*}
\label{thm:1}
\qed
\end{theorem}

As a result, the optimal policy for the current task can be obtained by computing 
$
\pistar(s) \in \argmax_{a \in \action} \left( \max_{g \in \goals} \qstarbar(s, g, a) \right).
$

Having established WVFs as a type of task-specific general value function (GVF) \citep{sutton11}, we next prove in Theorem~\ref{thm:mastery} that they do indeed have mastery---that is, they learn how to reach all achievable goal states in the world. We first formally define mastery as follows:

\begin{definition}
Let $\qstarbar$ be the optimal world action-value function for a task $M$. Then $\qstarbar$ has mastery if for all  $g \in \goals$ reachable from $s \in \state \setminus \{g\}$, there exists an optimal world policy
$
\pibarstar(s,g) \in \argmax\limits_{a \in \action}\qstarbar(s, g, a)
$ 
such that 
$
\pibarstar \in \argmax\limits_{\pibar} P^{\pibar}_{s}(s_{T} = g),
$
where $P^{\pibar}_{s}(s_{T} = g)$ is the probability of reaching $g$ from $s$ under a policy $\pibar$.
\label{def:mastery}
\end{definition}

\begin{theorem}
\label{theorem:2}
Let $\qstarbar$ be the optimal world action-value function for a task $M$. Then $\qstarbar$ has mastery.
\label{thm:mastery}
\qed
\end{theorem}

\SetAlgoNoLine
\begin{algorithm}
\label{alg:dq}
\DontPrintSemicolon
    \SetKwInOut{Initialise}{Initialise}
 \Initialise{ WVF $\qbar$, goal buffer $\goals$, learning rate $\alpha$ \;}
\ForEach{episode}{
    Observe initial state $s\in\state$ and sample $g \in \goals$\; 
   \While{episode is not done}{
    $a \gets 
    \begin{cases}
    \argmax\limits_{a \in \action} \qbar(s, g, a) & \mbox{w.p.  } 1 - \varepsilon  \\
    \text{a random action} & \mbox{w.p. } \varepsilon 
    \end{cases}$ \;
   Execute $a$, observe reward $r$ and next state $s^\prime$ \;
   \textbf{if} \textit{$s^\prime$ is absorbing} \textbf{then}  $\goals \leftarrow \goals \cup \{s\}$ \;
   \For{$g^\prime\in\goals$}{
    $\bar{r} \gets \rbarmin$ \textbf{if} $g^\prime \neq s$ and $s \in \goals$ \textbf{else} $r$ \;
    $\delta \gets \left[ \bar{r} + \max\limits_{a^\prime} \qbar(s^\prime, g^\prime, a^\prime) \right] - \qbar(s, g^\prime, a)$\;
    $\qbar(s, g^\prime, a) \gets \qbar(s, g^\prime, a) + \alpha \delta$\;
    }
    $s \leftarrow s^\prime$
   }
 }
 \caption{Q-learning for WVFs}
 \label{algo:qlearn}
\end{algorithm}

Finally, we note that while GVFs can also be used to construct goal-oriented value functions, questions remain open as to the origins of goals and how to define goal-specific rewards. 
WVFs are a subset of GVFs that answer these questions---goals are simply states with terminal transitions, while goal rewards are specified by $\rbar$. 
Answering these questions in this way confers several advantages, which we desribe below. 

\subsection{Planning with World Value Functions}

If the agent's goal space coincides with the state space ($\goals=\state$), then an optimal WVF will implicitly encode the dynamics of the world.
We can then estimate the transition probabilities for each $s,a\in\state\times\action$ using only the reward function and optimal WVF. That is, $P(s,a,s')$ for all $s^\prime\in\state$ can be obtained by simply solving the system of Bellman optimality equations given by each goal $ g\in\state$:  
$
\qstarbar(s,g,a) = \sum_{s' \in \state} p(s,a,s^\prime) \left[ \rbar(s,g,a,s') + \vstarbar(s',g) \right].
$
In practice, if the transition probabilities are known to be non-zero only in a neighbourhood $\mathcal{N}(s)$ of state $s$ (as is common in most domains), then we only require that the WVF be near-optimal for  $s^\prime,g\in\mathcal{N}(s)\times\mathcal{N}(s)$.

\subsection{Multitask Transfer with World Value Functions}

We now show the advantage of WVFs under the assumption that an agent may be faced with solving several tasks within the same world.
In other words, we assume that all tasks share the same state space, action space and dynamics, but differ in their reward functions. 
Formally, we define the world as a background MDP $M_0 = (\mathcal{\state}_0, \action_0, \dynamics_0, \reward_0 )$ with its own state space, action space, transition dynamics and background reward function. 
Any individual task $M$ is defined by a reward function $\reward_M^\tau(s,a)$ that is non-zero only for transitions entering terminal states.
The reward function for the resulting MDP is then simply $\reward_M(s,a,s^\prime) \coloneqq \reward_0(s,a,s^\prime) + \reward_M^\tau(s,a)$. 
We denote the set of all such tasks as $\tasks$, and the corresponding set of optimal WVFs as $\goalq$.

One immediate result is that if tasks share the same background MDP, then their WVFs share the same world policy.
That is, the agent has the same notion of goals and how to reach them, regardless of the current  task.
Similarly, if we require that the world policies be the same across tasks, then we have that the tasks must come from the same world.
This is formalised by Theorem~\ref{thm:pi1_e_pi2} below.

\begin{theorem} 
Let $\goalq$ be the set of optimal world $\bar{Q}$-value functions with mastery of tasks in $\tasks$. Then for all $s \neq g \in \state\times\goals$,
\[
\pibarstar(s,g) \in \argmax\limits_{a \in \action}\qstarbar_{M_1}(s, g, a) 
\]
\[
\iff 
\]
\[
\pibarstar(s,g) \in \argmax\limits_{a \in \action}\qstarbar_{M_2}(s, g, a) ~ \forall M_1, M_2 \in \tasks.
\]
\label{thm:pi1_e_pi2}
\qed
\end{theorem}



Since all tasks in $\tasks$ share the same dynamics (and consequently the same world policy), their corresponding WVFs can be written as $\qstarbar_M(s, g, a) = \gstar + \rbar_M^\tau(s^\prime, a^\prime)$ for some $s^\prime,a^\prime \in \state\times\action$, where $\gstar$ is a constant across tasks that represents the sum of rewards starting from $s$ and taking action $a$ up until $g$, but not including the terminal reward. 
Using this fact, Theorem~\ref{thm:R_WVF} shows that the optimal value function and policy for any task can be obtained zero-shot from an arbitrary WVF given the task-specific rewards:

\begin{theorem}
Let $\reward_M^\tau$ be the given task-specific reward function for a task $M\in\tasks$, and let $\qstarbar\in\goalq$ be an arbitrary WVF. Let $\vstarbara_M(s,g)$ be the estimated WVF of $M$ given by 
\[
\max\limits_{a\in\action}\qstarbar(s,g,a) + \left(\max\limits_{a\in\action}\reward_M^\tau(g,a)-\max\limits_{a\in\action}\qstarbar(g,g,a)\right).
\]
Then,
\begin{enumerate*}[label=(\roman*)]
    \item for all $ g\in\goals$ reachable from $s\in\state$, $\vstarbar_M(s,g) = \vstarbara_M(s,g)$.
    \item $\vstar_M(s) = \max\limits_{g\in\goals} \vstarbara(s,g)$, and $\pistar_M(s) \in \argmax\limits_{a\in\action}\qstarbar(s,\argmax\limits_{g\in\goals} \vstarbara_M(s,g),a)$.
\end{enumerate*}

\label{thm:R_WVF}
\qed
\end{theorem}

This has several important implications for transfer learning. 
Most importantly, an agent can learn an arbitrary WVF with unsupervised pretraining and then solve any new task by simply estimating the reward function (from experience or demonstrations).

\section{Experiments}

We empirically validate the properties of WVFs in the Four Rooms domain \citep{sutton99}, where an agent is required to reach various goal positions. 
The agent can move in any of the four cardinal directions at each timestep (with reward $-0.1$), but colliding with a wall leaves it in the same state. The agent also has a ``done'' action that can choose to terminate at any position (with reward $10$ if it is the goal of the current task).
For each of the experiments below, we consider the case where the agent's goals are the entire state space ($\goals = \state$). 

\subsection{Learning World Value Functions}

To verify that WVFs can be learned with standard model-free algorithms, we train an agent using Q-learning on a task where it must learn to navigate to either the middle of the top-left or bottom-right rooms. 
Figure~\ref{fig:WVF1} shows the learned WVF, which is generated by plotting the value functions for every goal position and displaying them at their respective $xy$ positions. 
Note how the values with respect the ``top-left'' and ``bottom-right'' goals are high (red), reflecting the high rewards the agent receives for reaching the goals it intended to achieve.
Figure~\ref{fig:WVF2} shows a close-up view of the learned WVF around the ``top-left'' goal. 
We can observe from the value gradient of the plots that the WVF does indeed learn how to reach all positions in the gridworld.
We can then maximise over goals to obtain the regular value function and policy (Figure~\ref{fig:WVF3}). 

Finally, we plot the returns obtained during the learning of both the WVF and regular value function, with results given by Figure~\ref{fig:WVF4}. 
Interestingly, this result indicates that it is more sample efficient to learn a WVF, despite the fact that it has an additional dimension that must be learned.
We theorise this is due to the induced goal-directed exploration of Algorithm~\ref{algo:qlearn}, which is far superior to $\varepsilon$-greedy exploration.

\begin{figure}[h!]
\centering
    \begin{subfigure}[b]{0.3\linewidth}
         \centering
         \includegraphics[width=\textwidth]{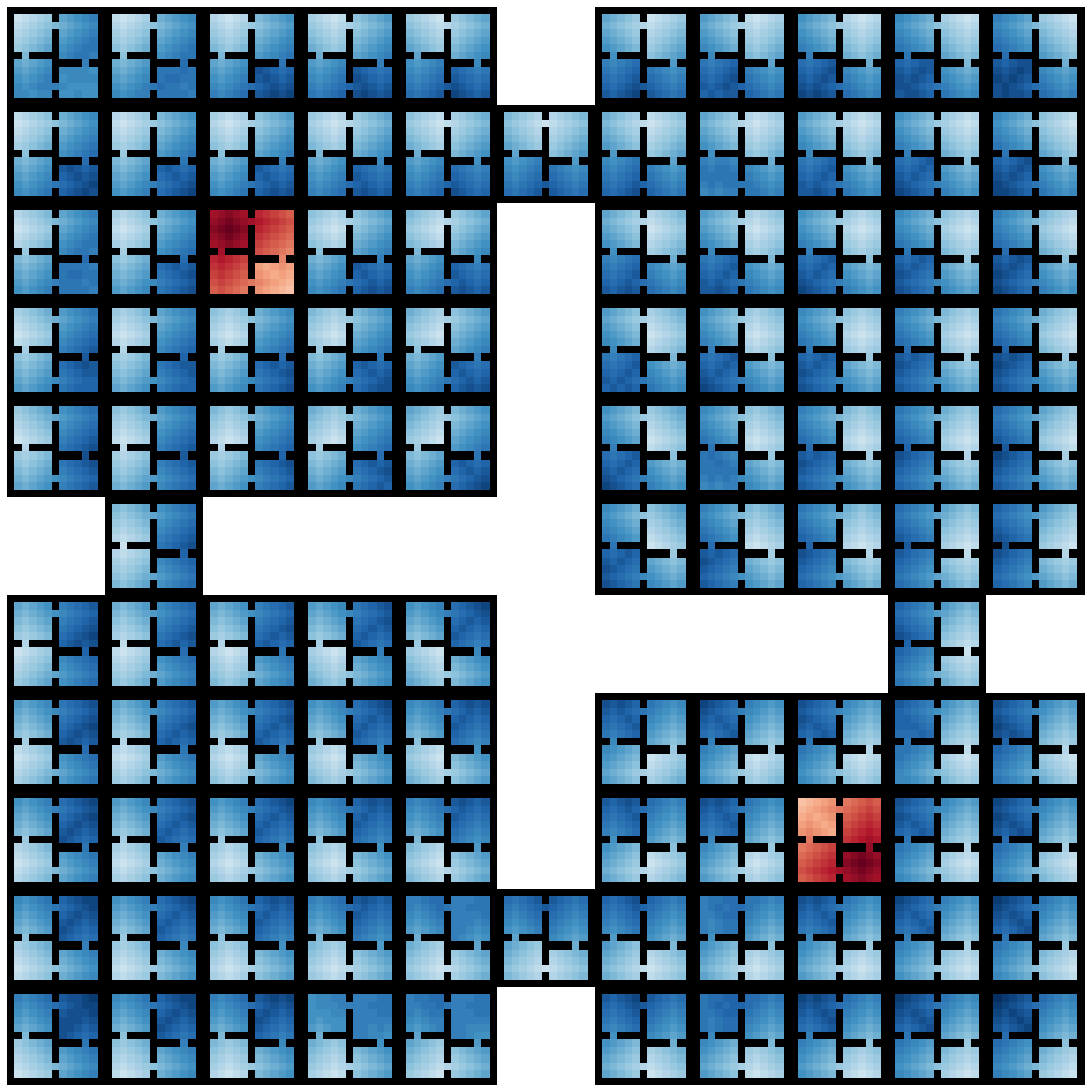}
         \caption{}
         \label{fig:WVF1}
     \end{subfigure}
     ~
     \begin{subfigure}[b]{0.3\linewidth}
         \centering
         \includegraphics[width=\textwidth]{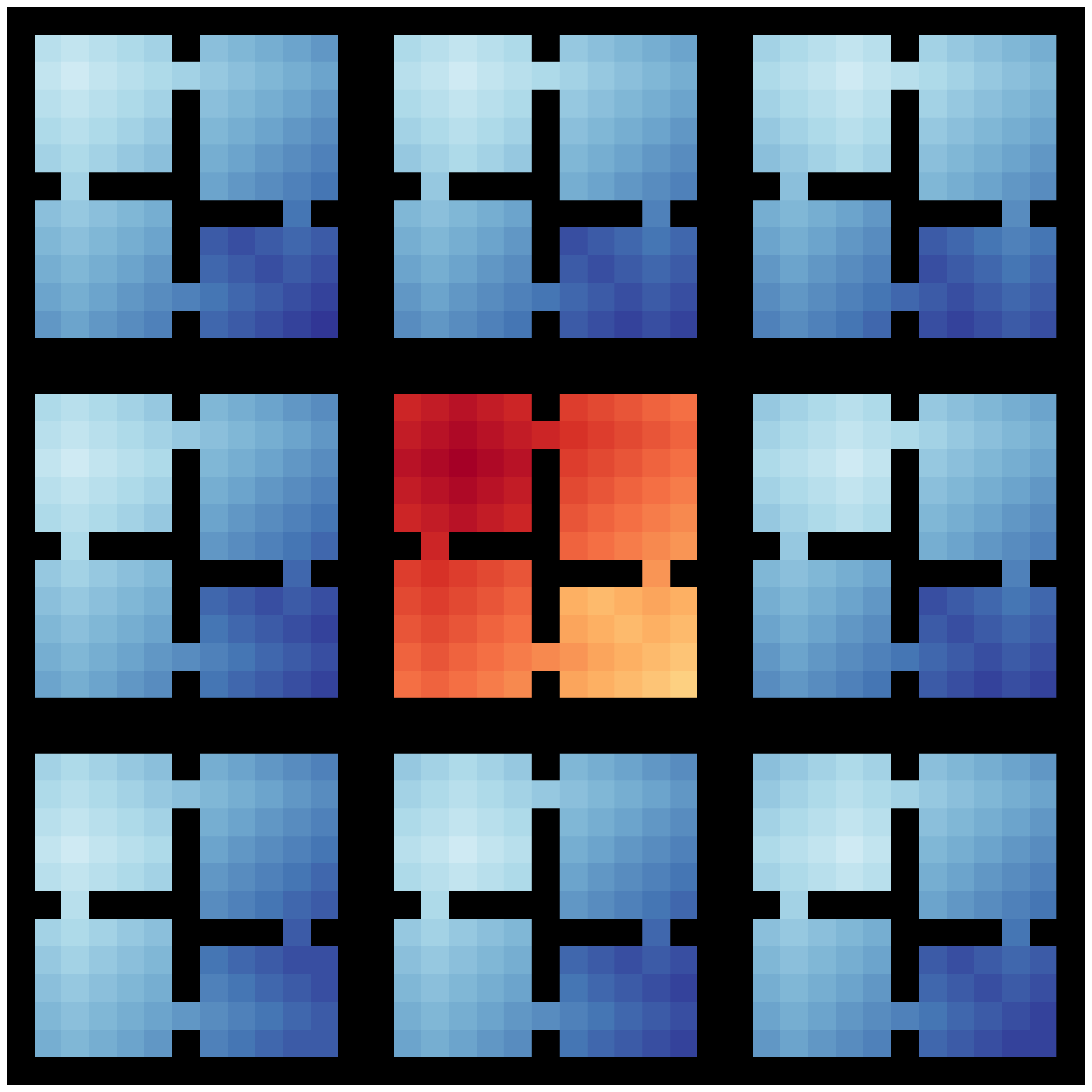}
         \caption{}
         \label{fig:WVF2}
     \end{subfigure}
     ~
     \begin{subfigure}[b]{0.3\linewidth}
         \centering
         \includegraphics[width=\textwidth]{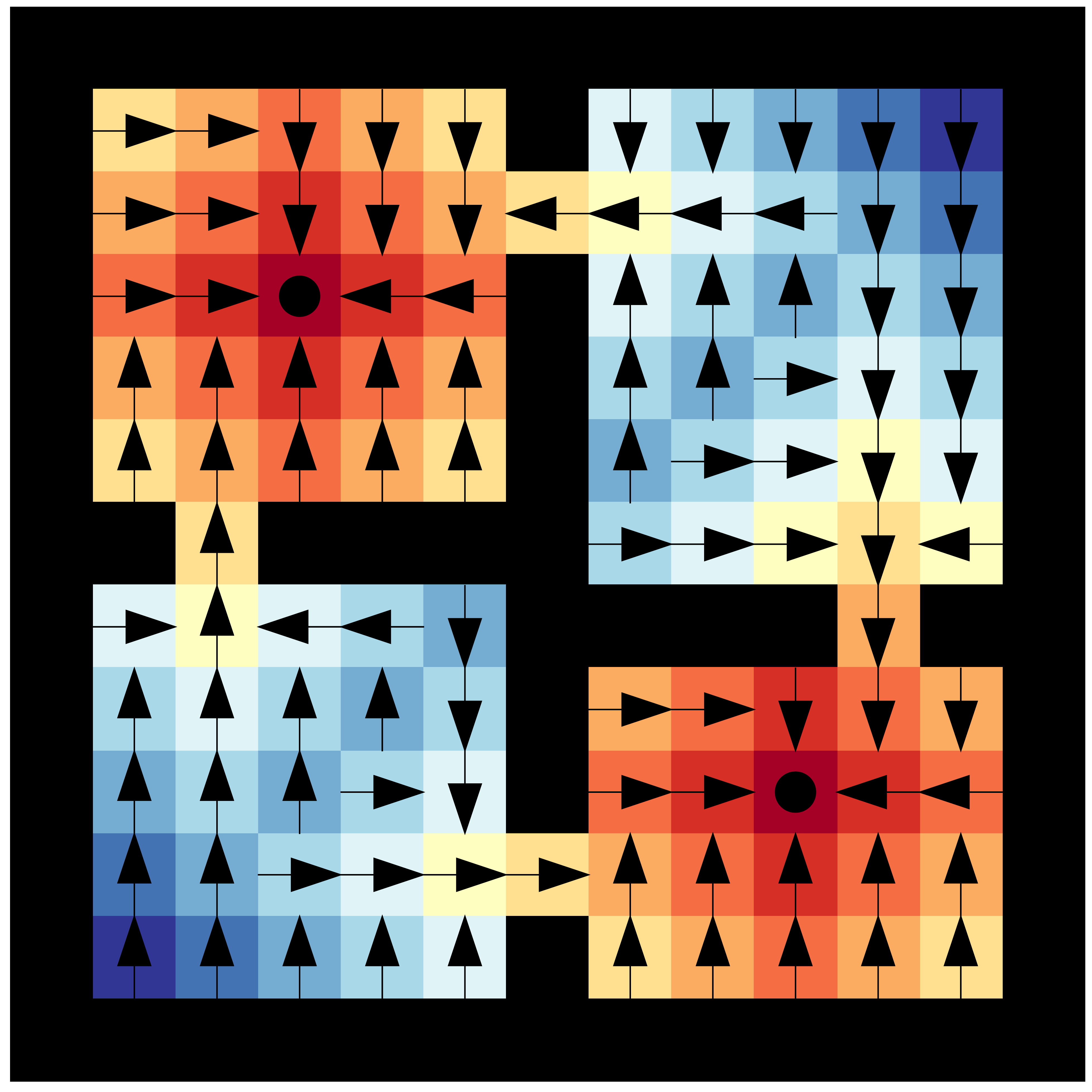}
         \caption{}
         \label{fig:WVF3}
     \end{subfigure}
     \\
     \centering
    \begin{subfigure}[b]{0.85\linewidth}
         \centering
         \includegraphics[width=\textwidth]{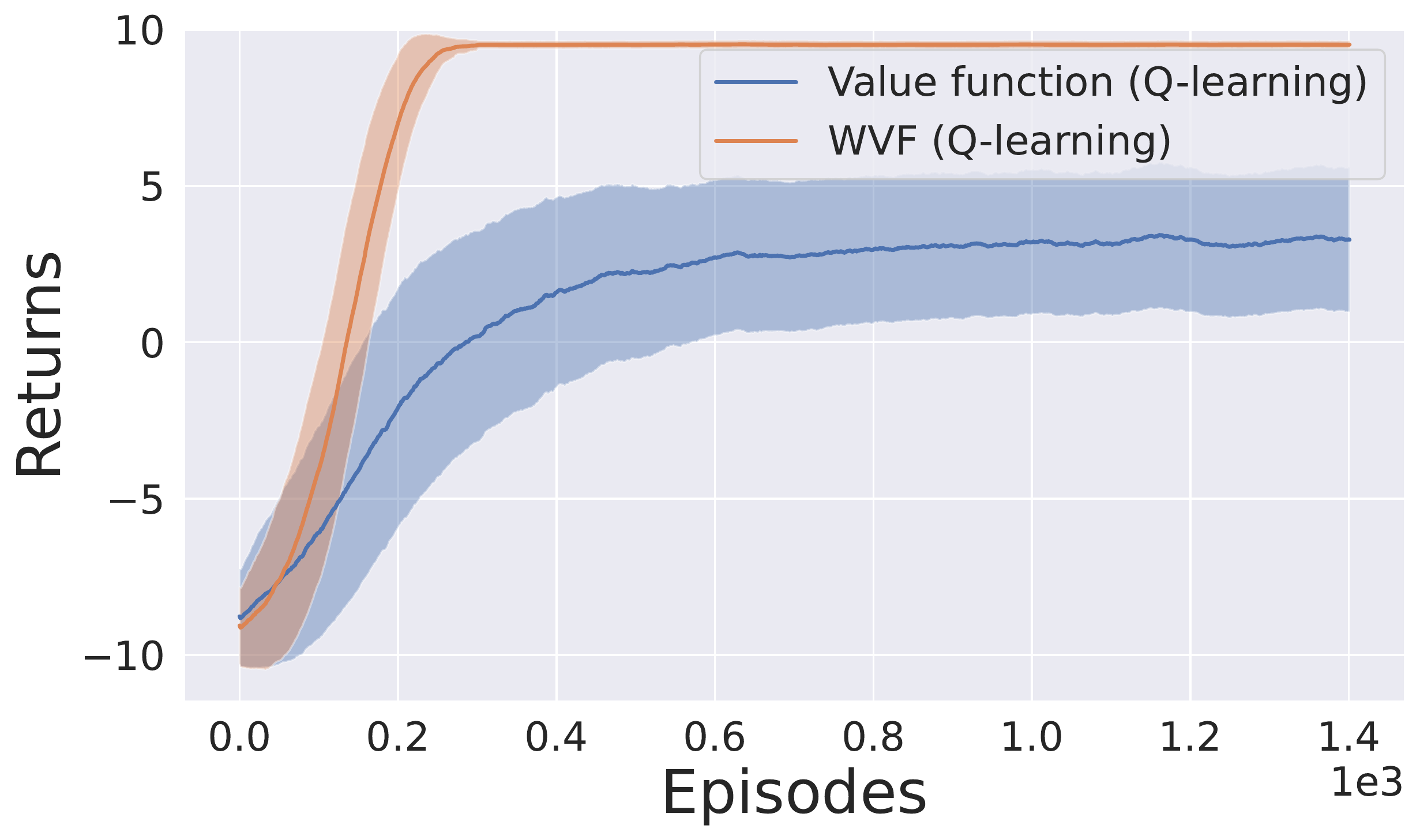}
         \caption{}
         \label{fig:WVF4}
     \end{subfigure}
     \caption{(a) Learned WVF. (b) Close-up view of the WVF for ``top-left'' goal. (c) Inferred values and policy for solving the current task. (d) Returns during training for both WVFs and regular value functions. Returns are calculated by greedy evaluation at the end of each episode. Mean and standard deviation over 25 random seeds are shown.}
     \label{fig:WVF}
\end{figure}

\begin{figure*}[t!]
\centering
\begin{subfigure}{.5\textwidth}
    \centering
    \includegraphics[height=1.1in]{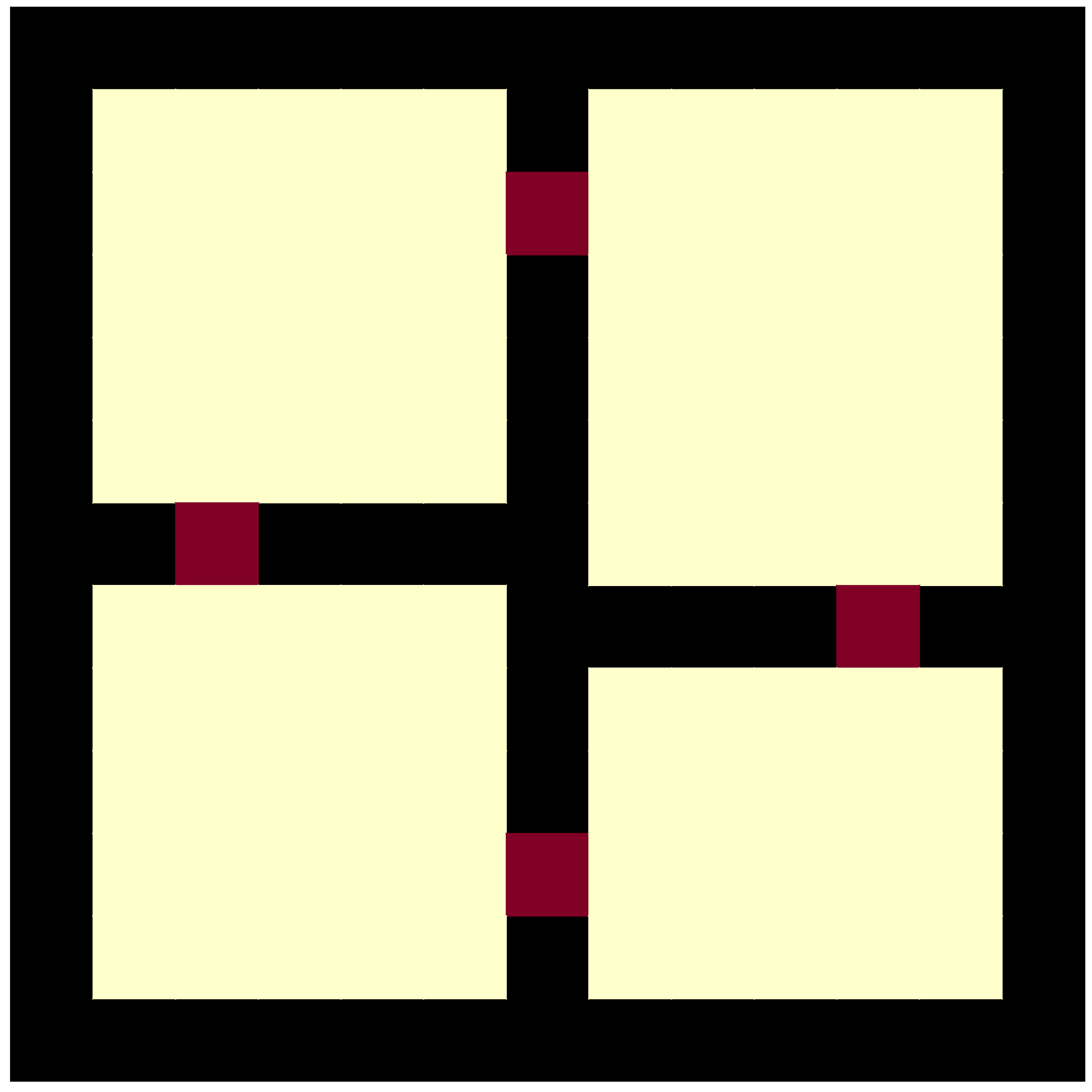}
    \includegraphics[height=1.1in]{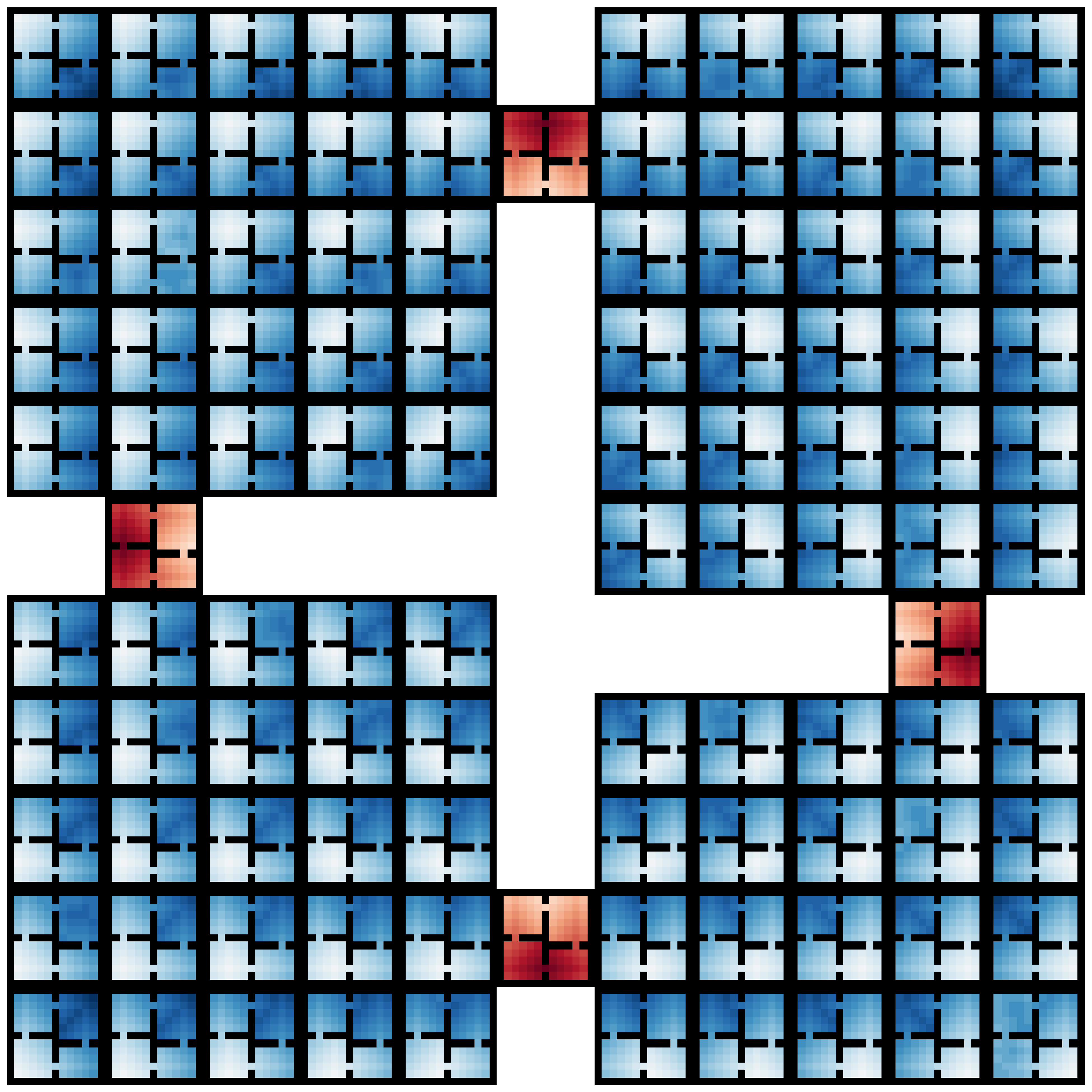}
    \includegraphics[height=1.1in]{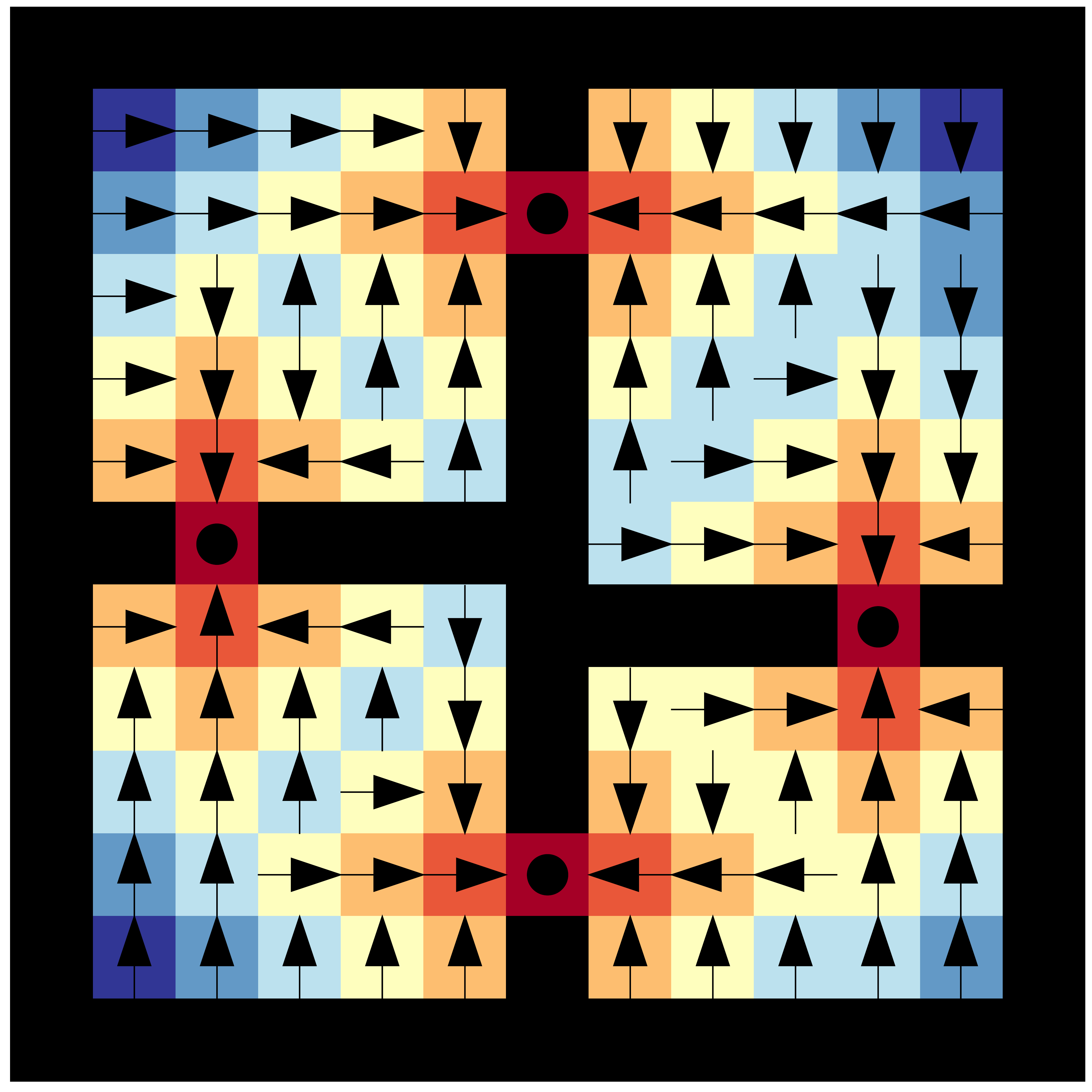}
    \caption{Navigating to the hallways.}
    \label{fig:task1}
\end{subfigure}%
\begin{subfigure}{.5\textwidth}
    \centering
    \includegraphics[height=1.1in]{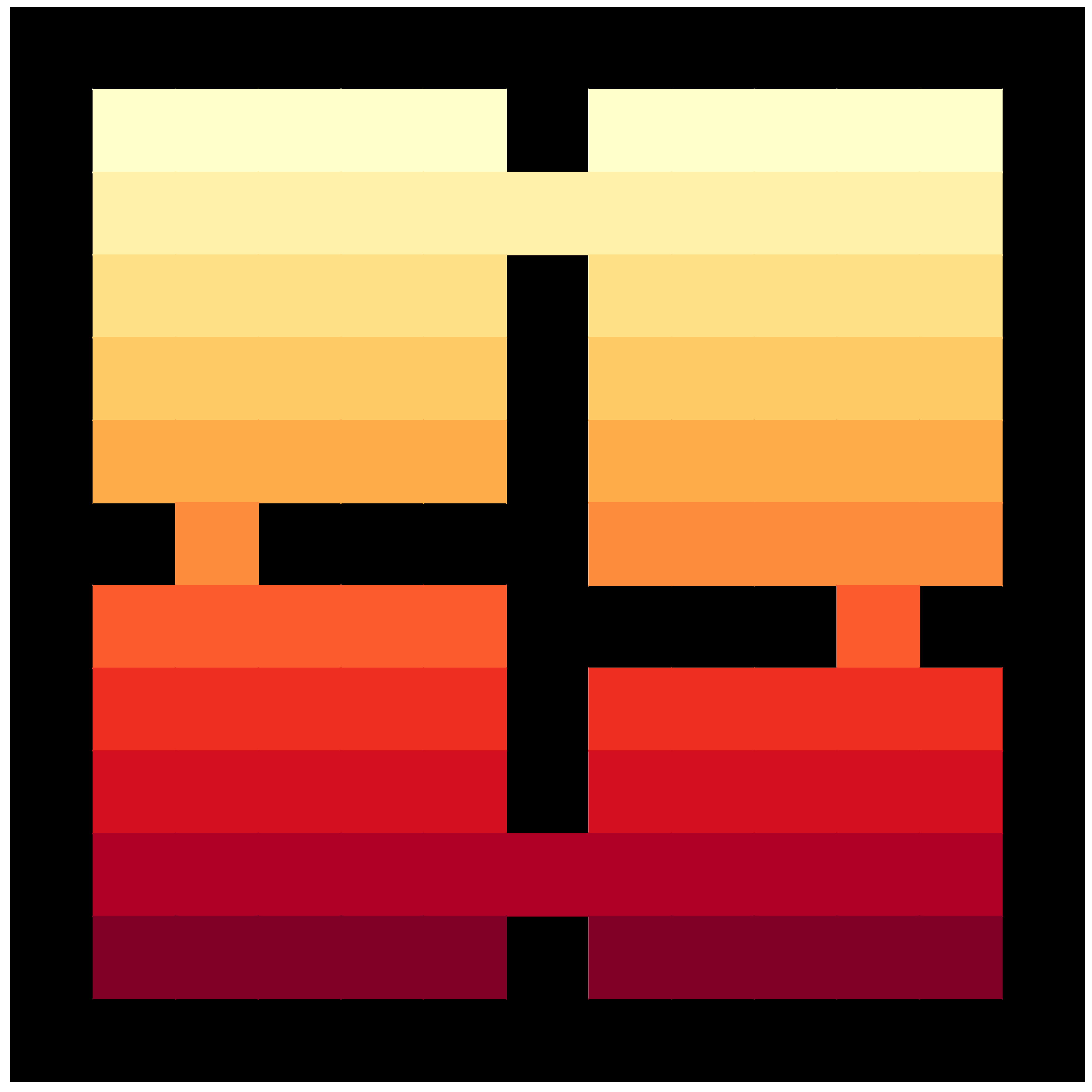}
    \includegraphics[height=1.1in]{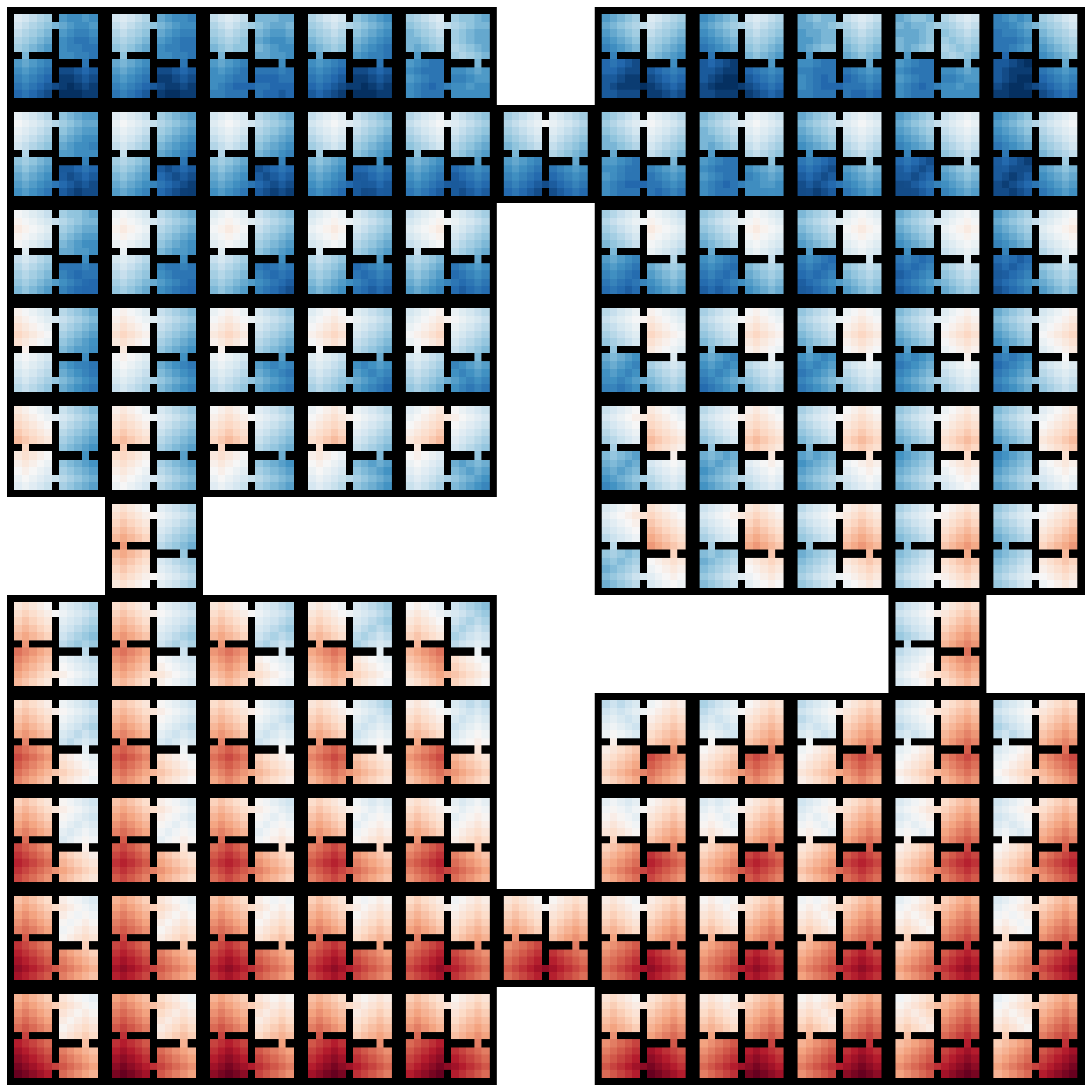}
    \includegraphics[height=1.1in]{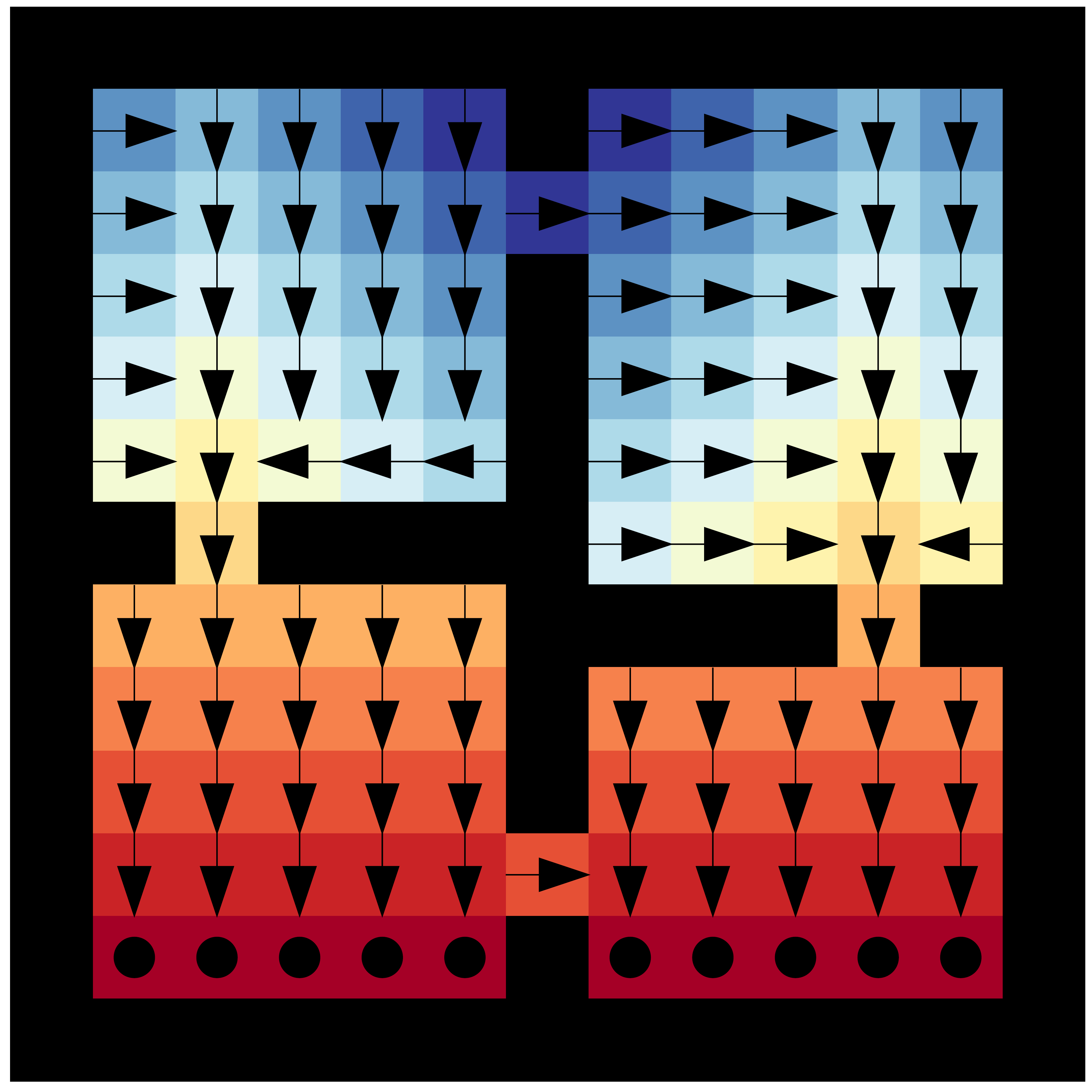}
    \caption{Navigating to the bottom of the grid.}
    \label{fig:task2}
\end{subfigure}%
\caption{From left to right on each figure: The task specific rewards, the inferred WVF using Theorem~\ref{thm:R_WVF}, and the inferred values and policy from maximising over goals for (a) reaching any of the hallways, and (b) reaching the bottom of the grid.}
\label{fig:R_WVF}
\end{figure*}

\subsection{Multitask Transfer with World Value Functions}

Having learned the WVF for the above task, we now show that it can be used to solve subsequent tasks by combining the WVF with the task-specific rewards as per Theorem~\ref{thm:R_WVF}. 
Critically, this means that any new task an agent might face can simply be solved by estimating its reward function, reducing the RL problem to a supervised learning one.
We consider two new tasks: navigating to any of the hallways, and navigating to the bottom of the grid. 
Figures~\ref{fig:task1} and \ref{fig:task2} illustrate the reward functions and subsequent WVFs and policies for these two tasks respectively.
Importantly, given the reward functions (which can be estimated from data), the optimal policies can immediately be computed without further learning. 

\subsection{Planning with World Value Functions}

Finally, we demonstrate that the transition probabilities can be inferred from the learned WVF. 
Figures~\ref{fig:transitions} (left) and (middle) respectively show the transitions inferred by solving the Bellman equations with $s^\prime,g\in\state\times\state$ and $s^\prime,g\in\mathcal{N}(s)\times\mathcal{N}(s)$.
For each, we infer the next state probabilities for taking each cardinal action at the center of each room, and place the corresponding arrow in the state with highest probability.
The red arrows in Figure~\ref{fig:transitions} (left) correspond to incorrectly inferred next states, which is a consequence of the learned WVF not being near optimal at all states for all goals.  
Figure~\ref{fig:transitions} (middle) shows that in practice, if the WVF is not near-optimal, we can still infer dynamics by using $s^\prime,g\in\mathcal{N}(s)\times\mathcal{N}(s)$.
Figure~\ref{fig:transitions} (right) shows sample trajectories for following the optimal policy using the inferred transition probabilities. The gray-scale color of each arrow corresponds to the normalised value prediction for that state.

Finally, we also demonstrate that these inferred dynamics can be used to improve planning by integrating WVFs into a Dyna-style architecture \citep{sutton90}. 
Our approach is illustrated by Algorithm~\ref{alg:dyna} in the Appendix, where we combine both model-free and model-based updates to learn the WVF. 
Importantly, since the dynamics are inferred from the WVF, using them to plan (Dyna-style) at the start of training is detrimental, since the WVF will make incorrect predictions. 
We mitigate this by computing the mean-squared error of the Bellman equations using the inferred next state, 
$MSE = \frac{1}{|\mathcal{N}(s)|}\sum_{g\in\mathcal{N}(s)}\left(\qbar(s,g,a)-\left[ \rbar(s,g,a,s^\prime) + \vbar(s^\prime,g) \right]\right)$,
and only use the WVF to plan when the error is less than a threshold
($MSE \leq 10^{-5}$).
We compare our approach to Q-learning for both WVFs and regular value functions, as well as Dyna for regular value functions. 
The results in Figure~\ref{fig:transitions4} illustrate that sample efficiency can be greatly improved by integrating the planning capabilities of WVFs.

\begin{figure}[h!]
\centering
    \begin{subfigure}[b]{0.3\linewidth}
         \centering
         \includegraphics[width=\textwidth]{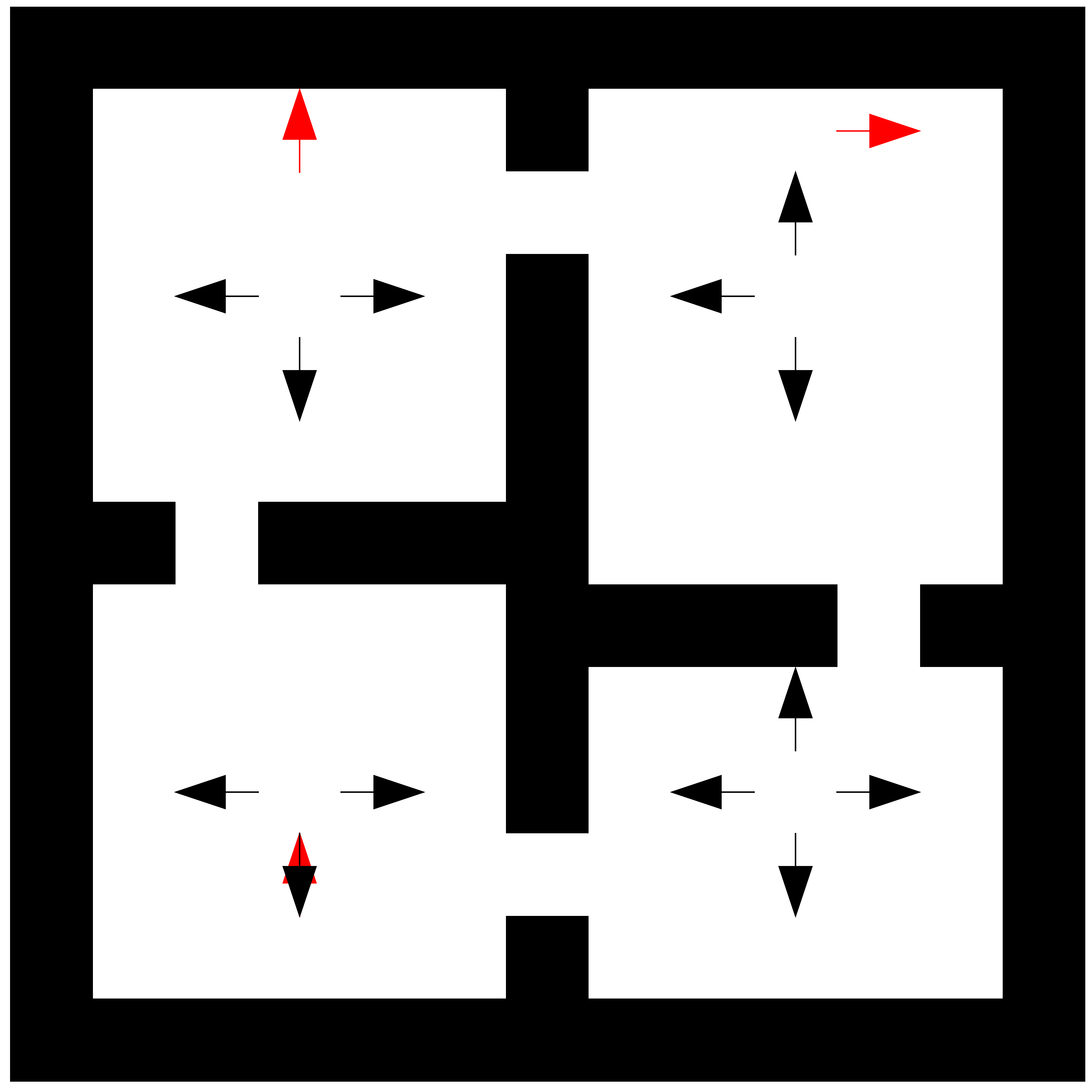}
         \caption{}
         \label{fig:transitions1}
     \end{subfigure}
     ~
     \begin{subfigure}[b]{0.3\linewidth}
         \centering
         \includegraphics[width=\textwidth]{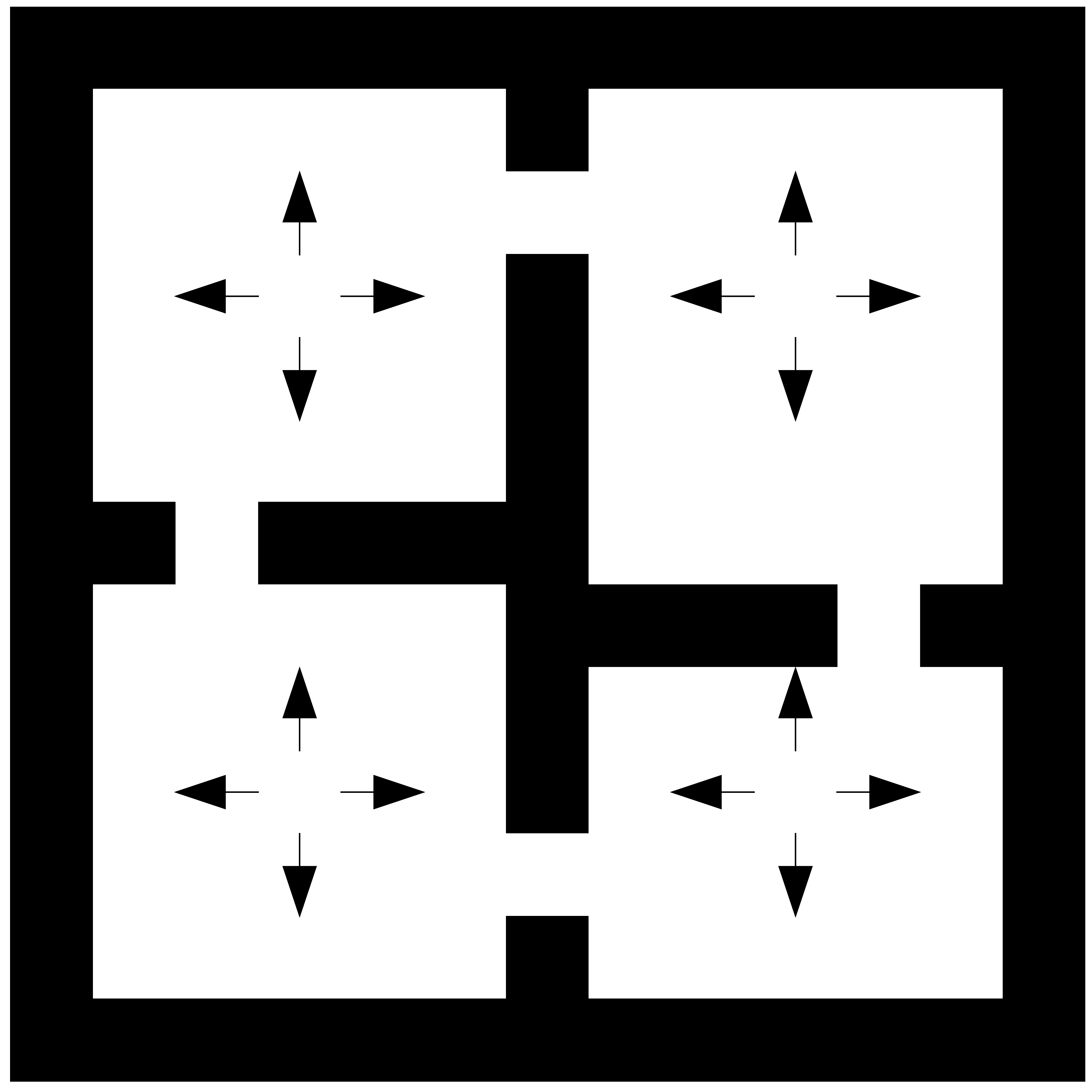}
         \caption{}
         \label{fig:transitions2}
     \end{subfigure}
     ~
     \begin{subfigure}[b]{0.3\linewidth}
         \centering
         \includegraphics[width=\textwidth]{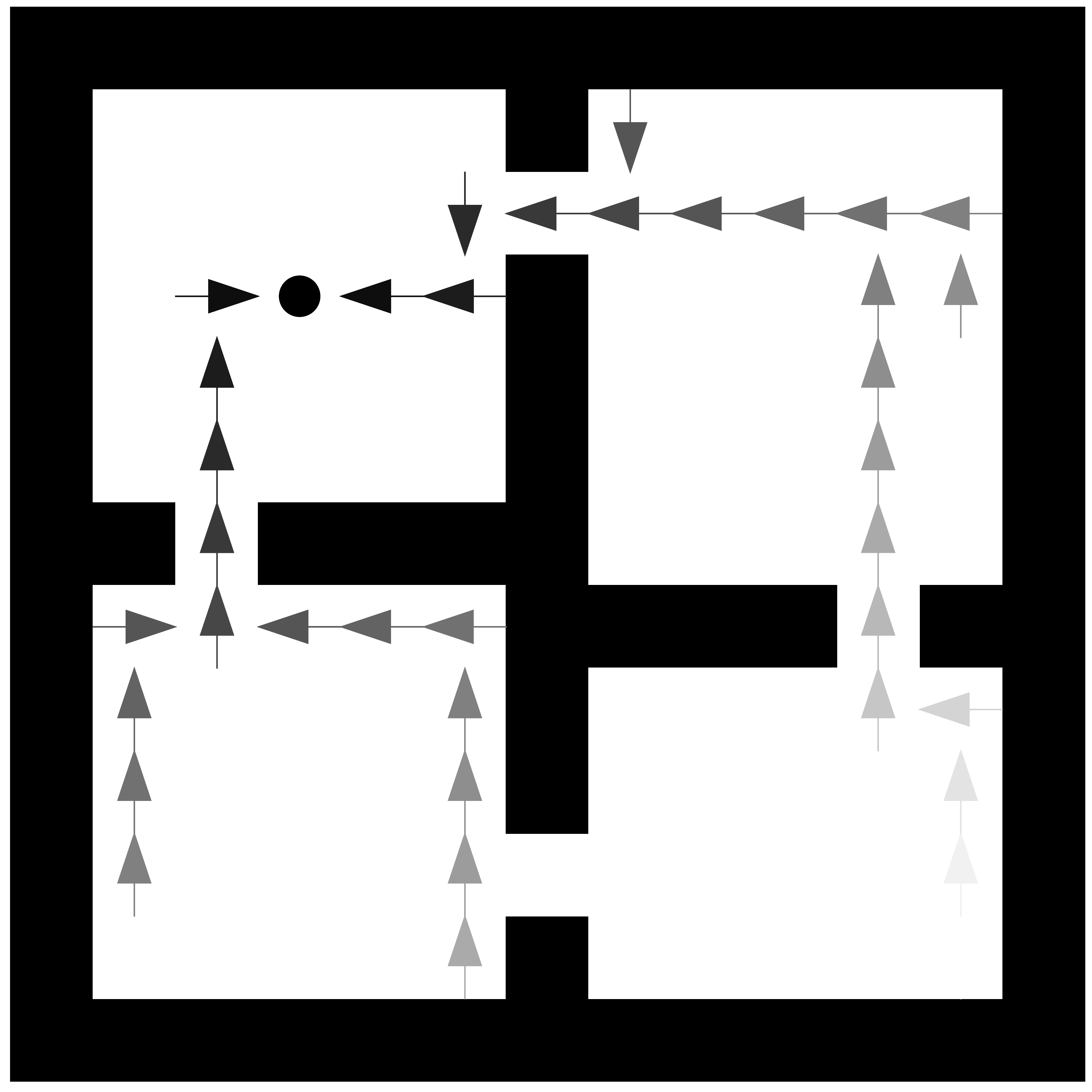}
         \caption{}
         \label{fig:transitions3}
     \end{subfigure}
     \\
     \centering
         \begin{subfigure}[b]{0.85\linewidth}
         \centering
         \includegraphics[width=\textwidth]{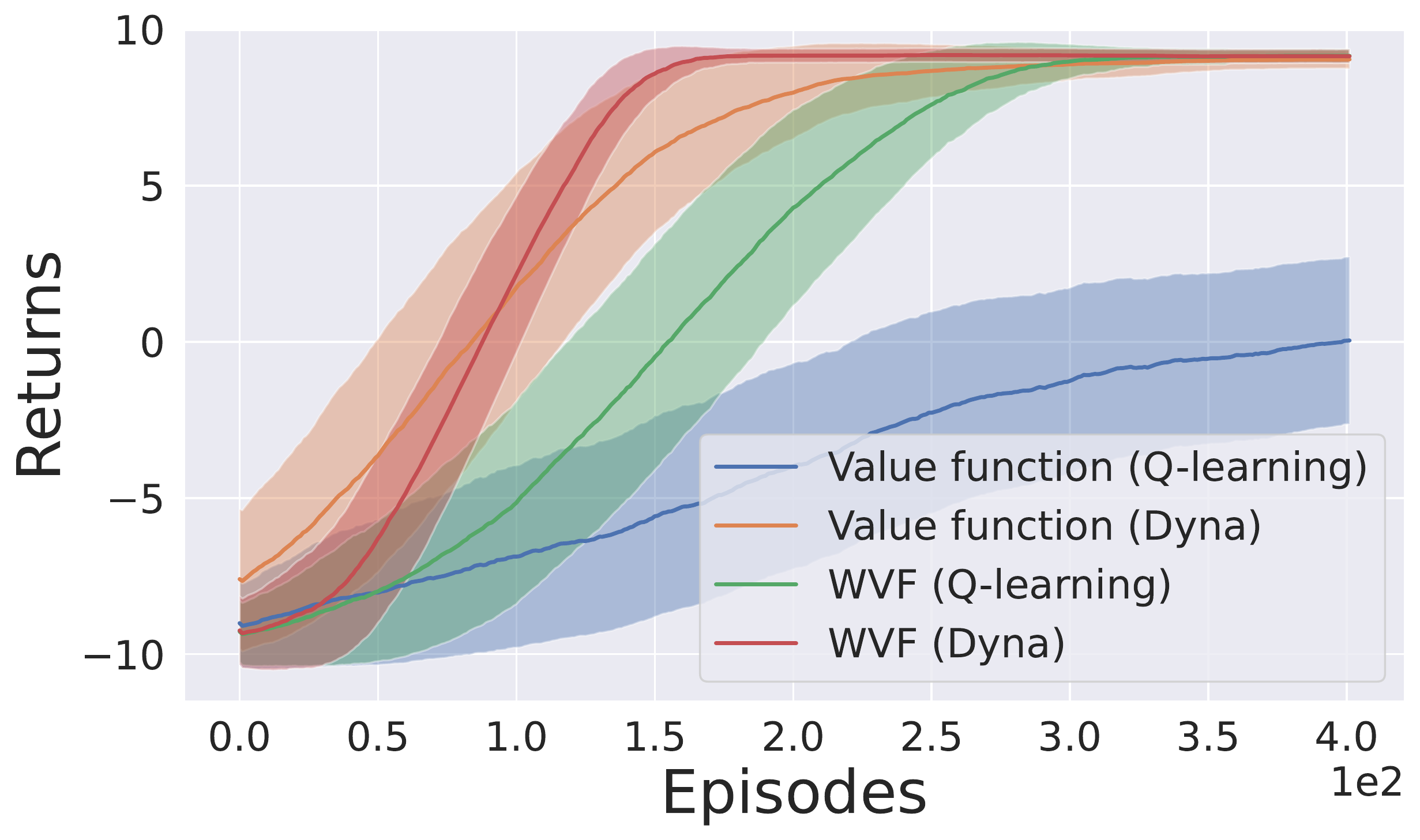}
         \caption{}
         \label{fig:transitions4}
     \end{subfigure}
     \caption{(a--b) Inferred one-step transitions. Red arrows indicate incorrect predictions. (c) Imagined rollouts using the learned WVF. (d) Returns during training for both WVFs and regular value functions, with and without planning. Mean and standard deviation over 25 random seeds are shown.}
     \label{fig:transitions}
\end{figure}


\section{Conclusion}

We introduced a new form of goal-oriented value function that encodes knowledge about how to solve all possible goal-reaching tasks in the world. 
This value function can be learned in a sample efficient manner, and can subsequently be used to infer the dynamics of the environment for model-based planning, or solve new tasks zero-shot given just their terminal rewards. 
An obvious path for future work is to extend these results to the stochastic high-dimensional setting. 
While prior work has demonstrated that WVFs can be learned with neural networks \citep{nangue2020boolean}, planning in high-dimensional environments is still an open challenge; WVFs may provide a promising avenue for unifying both learning and planning in this space.
Overall, our work is a step towards more general agents capable of solving any new task they may encounter.


\bibliography{aaai22}


\appendix

\section{Proofs of theoretical results}
\label{appendix:proofs}

\setcounter{theorem}{0}

\begin{theorem}
Let $M=(\mathcal{\state}, \action, \dynamics, \reward)$ be a deterministic task with optimal action-value function $\qstar$ and optimal world action-value function $\qstarbar$. 
Then for all $(s, a, s^\prime)$ in $\state \times \action \times \state$, we have
\begin{enumerate*}[label=(\roman*)]
    \item $\reward(s, a, s^\prime) = \max\limits_{g \in \goals} \rbar(s, g, a, s^\prime)$, and
    \item $\qstar(s, a) = \max\limits_{g \in \goals} \qstarbar(s, g, a)$.
\end{enumerate*}
\label{thm:1}
\end{theorem}
\begin{proof}
~
\begin{description}[align=right,leftmargin=*,labelindent=\widthof{(ii):}]
    \item [(i):]
        \begin{align*}
        &\max\limits_{g \in \goals} \rbar_M(s, g, a, s^\prime) \\
        &= 
            \begin{cases}
                    \max\{\rbarmin, \reward_{M}(s, a, s^\prime)\}, & \text{if } s \in \goals \\
                    \max\limits_{g \in \goals} \reward_M(s,a, s^\prime), & \text{otherwise.}
            \end{cases} \\
        &= \reward_{M}(s, a, s^\prime) \\
        &\text{(} \rbarmin \leq \rmin \leq \reward_{M}(s, a, s^\prime) \text{ by definition)}.
    \end{align*}
    \item [(ii):] Each $g$ in $\goals$ can be thought of as defining an MDP $M_g \vcentcolon = (\state, \action, \dynamics, \reward_{M_g})$ with reward function $\reward_{M_g}(s, a, s^\prime) \vcentcolon = \rbar_M(s, g, a, s^\prime)$ and optimal action-value function $\qstar_{M_g}(s, a) = \qstarbar_M(s, g, a)$. Then using (i) we have $\reward_M(s, a, s^\prime) =  \max\limits_{g \in \goals} \reward_{M_g}(s, a, s^\prime)$ and from \citet[Corollary 1]{vanniekerk19}, we have that $\qstar_M(s, a) =  \max\limits_{g \in \goals} \qstar_{M_g}(s, a) =  \max\limits_{g \in \goals} \qstarbar_M(s, g, a)$.
\end{description}
\end{proof}

\begin{theorem}
\label{theorem:2}
Let $\qstarbar$ be the optimal world action-value function for a task $M$. Then $\qstarbar$ has mastery.
\label{thm:mastery}
\end{theorem}
\begin{proof}
Let each $g$ in $\goals$ define an MDP $M_{g}$ with reward function
\[
\reward_{M_g} \vcentcolon = \rbar_{M}(s,g,a,s^\prime)
\]
for all $(s,a,s^\prime)$ in $\state \times \action \times \state$. Define
\[ 
\pistar_g(s) \in \argmax\limits_{a \in \action} \qstar_{M, g}(s,a) \text{ for all } s \in \state. 
\]
If $g$ \textit{is} reachable from $s \in \state \setminus \{g\}$, then we show that following $\pistar_g$ must reach $g$.
Since $\pistar_g$ is proper, it must reach a state $g^\prime \in \goals$ such that the transition $(g^\prime, \pistar_g(g^\prime),s^\prime)$ is terminal.
Assume $g^\prime \neq g$. 
Let $\pi_g$ be a policy that produces the shortest trajectory to $g$. 
Let $G^{\pistar_g}$ and $G^{\pi_g}$ be the returns for the respective policies. Then,
\begin{align*}
    &G^{\pistar_g} \geq G^{\pi_g} \\
    &\implies G^{\pistar_g}_{T-1} + \reward_{M_g}(g^\prime, \pistar_g(g^\prime),s^\prime) \geq G^{\pi_g}, \\ 
    &\text{ where } G^{\pistar_g}_{T-1} = \sum_{t=0}^{T-1}\reward_{M_g}(s_t,\pistar_g(s_t),s_{t+1}) \\
    &\text{ and } T \text{ is the time at which } g^\prime \text{ is reached.}     \\
    &\implies G^{\pistar_g}_{T-1} + \rbarmin \geq G^{\pi_g}, \text{ since } g \neq g^\prime \in \goals \\
    &\implies \rbarmin \geq G^{\pi_g} - G^{\pistar_g}_{T-1} \\
    &\implies (\rmin - \rmax)D \geq G^{\pi_g} - G^{\pistar_g}_{T-1} , \\
    &\quad \text{ by definition of } \rbarmin \\
    &\implies G^{\pistar_g}_{T-1} - \rmax D \geq G^{\pi_g} - \rmin D, \\
    &\quad \text{ since } G^{\pi_g} \geq \rmin D \\
    &\implies G^{\pistar_g}_{T-1} - \rmax D \geq 0 \\
    &\implies G^{\pistar_g}_{T-1} \geq \rmax D.
\end{align*}
But this is a contradiction, since the result obtained by following an optimal trajectory up to a terminal state without the reward for entering the terminal state must be strictly less than receiving $\rmax$ for every step of the longest possible optimal trajectory. 
Hence we must have $g^\prime = g$. 

\end{proof}

\newpage

\begin{theorem} 
Let $\goalq$ be the set of optimal world $\bar{Q}$-value functions with mastery of tasks in $\tasks$. Then for all $s \neq g \in \state\times\goals$,
\[
\pibarstar(s,g) \in \argmax\limits_{a \in \action}\qstarbar_{M_1}(s, g, a) 
\]
\[
\iff 
\]
\[
\pibarstar(s,g) \in \argmax\limits_{a \in \action}\qstarbar_{M_2}(s, g, a) ~ \forall M_1, M_2 \in \tasks.
\]
\label{thm:pi1_e_pi2}
\end{theorem}
\begin{proof}
Let $g \in \goals, s \in \state \setminus \{g\}$.

If $g$ \textit{is} reachable from $s$, then we are done since $\qstarbar_{M_1}$ and $\qstarbar_{M_2}$ have mastery (Theorem~\ref{thm:mastery}).

If $g$ is unreachable from $s$, then for all $(a, s^\prime)$ in $\action \times \state$ we have 

\begin{align*}
    \rbar_{M_1}(s, g, a, s^\prime) &= \begin{cases}
    \rbarmin, &\text{ if } s^\prime \text{ is absorbing } \\
    r_{s,a, s^\prime}, &\text{ otherwise}
    \end{cases}\\
    &\text{ where } r_{s,a, s^\prime} \text{ is the reward for the}\\
    &\text{ non-terminal transition } (s,a, s^\prime)\\
    &= \rbar_{M_2}(s, g, a, s^\prime) \\
    &\implies \qstarbar_{M_1}(s, g, a) = \qstarbar_{M_2}(s, g, a).
\end{align*}

\end{proof}

\begin{theorem}
Let $\reward_M^\tau$ be the given task-specific reward function for a task $M\in\tasks$, and let $\qstarbar\in\goalq$ be an arbitrary WVF. Let $\vstarbara_M(s,g)$ be the estimated WVF of $M$ given by 
\[
\max\limits_{a\in\action}\qstarbar(s,g,a) + \left(\max\limits_{a\in\action}\reward_M^\tau(g,a)-\max\limits_{a\in\action}\qstarbar(g,g,a)\right).
\]
Then,
\begin{enumerate}[label=(\roman*)]
    \item for all $g\in\goals$ reachable from $s\in\state$, $\vstarbar_M(s,g) = \vstarbara_M(s,g)$.
    \item $\vstar_M(s) = \max\limits_{g\in\goals} \vstarbara(s,g)$, and $\pistar_M(s) \in \argmax\limits_{a\in\action}\qstarbar(s,\argmax\limits_{g\in\goals} \vstarbara_M(s,g),a)$.
\end{enumerate}

\label{thm:R_WVF}
\end{theorem}
\begin{proof}
~
\begin{description}[align=right,leftmargin=*,labelindent=\widthof{(ii):}]
    \item [(i):] Let $g \in \goals$ be a goal reachable from state $s \in \state$. If $g=s$, then
\begin{align*}
    &\max\limits_{a\in\action}\qstarbar(s,g,a) + \left(\max\limits_{a\in\action}\reward_M^\tau(g,a)-\max\limits_{a\in\action}\qstarbar(g,g,a)\right) \\
    &= \max\limits_{a\in\action}\reward^\tau(g,a) + \left(\max\limits_{a\in\action}\reward_M^\tau(g,a)-\max\limits_{a\in\action}\reward^\tau(g,a)\right) \\
    &= \max\limits_{a\in\action}\reward_M^\tau(g,a) = \vstarbar_M(s,g) \\
\end{align*}

If $g\neq s$, then
\begin{align*}
    &\max\limits_{a\in\action}\qstarbar(s,g,a) + \left(\max\limits_{a\in\action}\reward_M^\tau(g,a)-\max\limits_{a\in\action}\qstarbar(g,g,a)\right) \\
    &= \max\limits_{a\in\action}[\gstar + \reward^\tau(g, a^{\max})] + \\ 
    &\quad \quad \quad \quad 
    \left(\max\limits_{a\in\action}\reward_M^\tau(g,a)-\max\limits_{a\in\action}\reward^\tau(g,a)\right), \\
    &\text{ follows from Theorem~\ref{thm:mastery} and Theorem~\ref{thm:pi1_e_pi2}} \\
    &= \max\limits_{a\in\action}\gstar + \reward^\tau(g, a^{\max}) + \\
    &\quad \quad \quad \quad
    \left(\reward_M^\tau(g,a_M^{\max})-\reward^\tau(g,a^{\max})\right) \\
    &= \max\limits_{a\in\action}[\gstar + \rbar_M^\tau(g, a_M^{\max})] = \vstarbar_M(s,g) \\
\end{align*}

    \item [(ii):]Follows directly from (i) above and Theorem~\ref{thm:pi1_e_pi2}.
\end{description}
\end{proof}

\section{Algorithms}

\SetAlgoNoLine
\begin{algorithm}
\DontPrintSemicolon
    \SetKwInOut{Initialise}{Initialise}
 \Initialise{ WVF $\qbar$, Reward function $\reward$, goal buffer $\goals$, learning rate $\alpha$ \;}
\ForEach{episode}{
    Observe initial state $s\in\state$ and sample $g \in \goals$\; 
   \While{episode is not done}{
    $a \gets 
    \begin{cases}
    \argmax\limits_{a \in \action} \qbar(s, g, a) & \mbox{w.p.  } 1 - \varepsilon  \\
    \text{a random action} & \mbox{w.p. } \varepsilon 
    \end{cases}$ \;
   Execute $a$, observe reward $r$ and next state $s^\prime$ \;
    $\reward(s,a, .) \gets r$ \;
   \textbf{if} \textit{$s^\prime$ is absorbing} \textbf{then}  $\goals \leftarrow \goals \cup \{s\}$ \;
   \For{$g^\prime\in\goals$}{
    $\bar{r} \gets \rbarmin$ \textbf{if} $g^\prime \neq s$ and $s \in \goals$ \textbf{else} $r$ \;
    $\delta \gets \left[ \bar{r} + \max\limits_{a^\prime} \qbar(s^\prime, g^\prime, a^\prime) \right] - \qbar(s, g^\prime, a)$\;
    $\qbar(s, g^\prime, a) \gets \qbar(s, g^\prime, a) + \alpha \delta$\;
    }
    
   \RepTimes{$N$}{
    $s \gets$ random previous state \;
    $a \gets$ random previous action taken in $s$ \;
    $r \gets \reward(s, a, .)$ \;
    $s^\prime \gets $ Solving $\mathcal{N}(s)$ Bellman equations \;
    $MSE \gets \frac{1}{|\mathcal{N}(s)|}\sum_{g\in\mathcal{N}(s)}(\qbar(s,g,a)-$ \;
    \quad \quad \quad \quad 
    $\left[ \rbar(s,g,a,s^\prime) + \vbar(s^\prime,g) \right])$ \;
    \If{$MSE \leq$ threshold}{
      \For{$g^\prime\in\goals$}{
            $\bar{r} \gets \rbarmin$ \textbf{if} $g^\prime \neq s$ and $s \in \goals$ \textbf{else} $r$ \;
            $\delta \gets \left[ \bar{r} + \max\limits_{a^\prime} \qbar(s^\prime, g^\prime, a^\prime) \right] - \qbar(s, g^\prime, a)$\;
            $\qbar(s, g^\prime, a) \gets \qbar(s, g^\prime, a) + \alpha \delta$\;
        }
    }
    }
    $s \leftarrow s^\prime$
   }
 }
 \caption{Dyna for WVFs using inferred transition functions}
 \label{alg:dyna}
\end{algorithm}

\end{document}